\documentclass[11pt]{article}
\usepackage[preprint]{acl}
\usepackage[T1]{fontenc}
\usepackage[utf8]{inputenc}
\usepackage{times}
\usepackage{latexsym}
\usepackage{microtype}
\usepackage{inconsolata}
\usepackage{amsmath,amssymb,amsfonts,amsthm}
\usepackage{algorithm}
\usepackage{algorithmic}
\usepackage{graphicx}
\usepackage{booktabs}
\usepackage{hyperref}
\usepackage{xcolor}
\usepackage{multirow}
\usepackage{enumitem}
\usepackage{array}
\usepackage{tabularx}
\usepackage{makecell}
\usepackage{float}
\usepackage{tikz}
\usepackage{pgfplots}
\usepackage{placeins}
\usetikzlibrary{arrows.meta,positioning,patterns,calc,fit,shapes}
\pgfplotsset{compat=1.17}
\pgfplotsset{
  aclplot/.style={
    width=\linewidth,
    height=4.8cm,
    grid=both,
    tick label style={font=\footnotesize},
    label style={font=\footnotesize},
    title style={font=\footnotesize}
  }
}

\newtheorem{theorem}{Theorem}
\newtheorem{lemma}{Lemma}
\newtheorem{corollary}{Corollary}
\newtheorem{definition}{Definition}
\newtheorem{assumption}{Assumption}

\newcommand{\score}{\widehat\Lambda}
\newcommand{\scr}{\mathrm{SCR}}

\title{Spectral-Guided Diffusion: Accelerating Inference\\via Static Spectral Layer Scheduling}

\author{
\textbf{Ibne Farabi Shihab}\textsuperscript{1}%
\thanks{Corresponding author: \texttt{ishihab@iastate.edu}.}
\and
\textbf{Abu Sa-Adat Mohamed Moon-Im Al Ahsan}\textsuperscript{2}
\and
\textbf{Anuj Sharma}\textsuperscript{3}
\\[2pt]
\textsuperscript{1}Department of Computer Science, Iowa State University \\
\textsuperscript{2}Department of Computer Science \& Engineering, BRAC University \\
\textsuperscript{3}Department of Civil, Construction \& Environmental Engineering, Iowa State University \\
\texttt{ishihab@iastate.edu}, \texttt{abu.sa.adat.mohamed.moon.im.al.ahshan@g.bracu.ac.bd}
}

\begin{document}
\maketitle

\begin{abstract}
Diffusion inference repeatedly evaluates the same large network. We ask whether pretrained weights alone can identify residual branches that need not be recomputed throughout the trajectory. Our \textbf{Spectral Concentration Ratio (SCR)} measures leading-versus-tail singular-value energy. Combined with Frobenius magnitude, it yields an offline sensitivity proxy and a deterministic lifetime for each scheduled unit. A frozen unit reuses its cached residual-branch update while the current residual stream and all external conditioning continue to propagate. The method needs no router, calibration prompts, or input-dependent search. At matched layer-step budgets, SCR/Frobenius preserves quality better than random, depth, norm, stable-rank, and Frobenius--stable-rank schedules on LLaDA-8B, DiT-XL/2, U-ViT-L, and SDXL. Broader LLaDA tests cover retrieval, reasoning, code, summarization, and open-ended generation; matched-horizon controls retain the ranking down to ten denoising steps. The complete captured-graph system reaches $2.8\times$--$3.0\times$ wall-clock speedup over eager inference. This is a systems-level number: on LLaDA, padded graph execution already gives $2.7\times$, while eliminating inactive branch work raises it to $3.0\times$. The perturbation analysis motivates pre-norm attention and MLP components under explicit local assumptions; results on AdaLN, U-shaped, convolutional, and cross-attention blocks are empirical transfer, not certified guarantees.
\end{abstract}

\section{Introduction}
\label{sec:intro}

Diffusion models generate strong samples across text, image, and audio, but they repeatedly evaluate a large denoiser~\citep{ho2020denoising, dhariwal2021diffusion, rombach2022high, saharia2022photorealistic, karras2022elucidating}. LLaDA-8B~\citep{nie2025llada} and SDXL~\citep{podell2024sdxl}, for example, require tens or hundreds of network evaluations per sample. Solvers reduce global steps~\citep{lu2023dpmsolver}; distillation compresses the trajectory~\citep{salimans2022progressive,song2023consistency}; and adaptive systems reuse computation through learned routing, calibration, or online decisions~\citep{ma2024deepcache,wimbauer2024cache,ma2024learningtocache,liu2025smoothcache,cao2026procache,chen2024delta}. We study a complementary question: how much useful scheduling information is already present in the pretrained weights?

We introduce the Spectral Concentration Ratio (SCR), a fixed leading-to-tail singular-energy statistic. SCR is not a knee detector: it measures energy in the top $\lfloor0.1d\rfloor$ singular directions relative to the remaining spectrum. Combining concentration with Frobenius magnitude gives a static sensitivity proxy. That proxy assigns each residual unit a lifetime; low-scoring branch updates are computed for fewer denoising iterations.

The schedule is computed once, reused for every input, and requires no router training, per-prompt logic, or calibration-time search. This determinism is useful only if the execution semantics are correct. We cache a residual-branch update, not a stale full hidden state. At a frozen unit, the current upstream state still passes through the identity path and receives the cached update. Dependency-coupled operations are scheduled as atomic groups, so active upstream computation is never discarded.

Our claim is deliberately narrower than input-specific optimality. The analysis identifies where spectral quantities enter local sensitivity for pre-norm attention and MLP branches, then exposes the drift and downstream-amplification terms that weights alone cannot determine. AdaLN, U-shaped skips, cross-attention, and convolution fall outside the direct argument. We test transfer to those components empirically and reserve the theorem-backed language for the stated assumptions.

Our contributions are fourfold. First, we define an SCR/Frobenius score and convert it into static lifetimes. Second, we give a conditional perturbation argument whose role is motivational, and audit its omitted non-weight factors post hoc. Third, we test the ranking at matched compute on four architectures, adding four LLaDA tasks, a short-horizon sweep, a Frobenius--stable-rank control, and an input-aware drift baseline. Fourth, a jagged CUDA-Graph executor converts the schedule into measured latency. The $2.8\times$--$3.0\times$ headline is the complete system relative to eager inference, not the causal effect of layer freezing alone. Figure~\ref{fig:overview} summarizes this offline--online split.

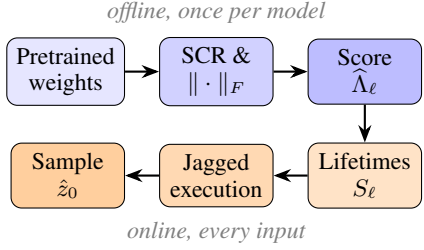
\begin{figure}[t]
\centering
\begin{tikzpicture}[
    node distance=0.45cm and 0.45cm,
    box/.style={draw, rounded corners, minimum width=1.5cm, minimum height=0.65cm, font=\small, align=center},
    label/.style={font=\footnotesize\itshape, gray},
    arrow/.style={-Stealth, thick}
]
\node[box, fill=blue!10] (weights) {Pretrained\\weights};
\node[box, right=of weights, fill=blue!18] (scr) {SCR \&\\$\|\cdot\|_F$};
\node[box, right=of scr, fill=blue!26] (score) {Score\\$\score_\ell$};
\node[box, below=0.5cm of score, fill=orange!22] (sched) {Lifetimes\\$S_\ell$};
\node[box, left=of sched, fill=orange!30] (exec) {Jagged\\execution};
\node[box, left=of exec, fill=orange!38] (out) {Sample\\$\hat z_0$};

\draw[arrow] (weights) -- (scr);
\draw[arrow] (scr) -- (score);
\draw[arrow] (score) -- (sched);
\draw[arrow] (sched) -- (exec);
\draw[arrow] (exec) -- (out);

\node[label, above=0.05cm of scr] {offline, once per model};
\node[label, below=0.05cm of exec] {online, every input};
\end{tikzpicture}
\caption{Spectral-Guided Scheduling. Offline weight spectra produce scores $\score_\ell$ and lifetimes $S_\ell$. Online, each frozen residual unit reuses a cached branch update; the current residual stream and external inputs still propagate.}
\label{fig:overview}
\end{figure}

\section{Related Work}
\label{sec:related}

DDIM~\citep{song2020denoising} and DPM-Solver++~\citep{lu2022dpm,lu2023dpmsolver} reduce denoising iterations. Progressive distillation~\citep{salimans2022progressive}, consistency models~\citep{song2023consistency,song2023ict}, and SnapFusion~\citep{li2023snapfusion} train models or architectures for shorter trajectories. These methods change the number of global steps; our scheduler changes branch work inside each retained step, so the mechanisms can compose.

Structural pruning and caching are closer to our setting. MosaicDiff~\citep{guo2025mosaicdiff} uses training-phase structure to allocate static sparsity. DeepCache~\citep{ma2024deepcache}, block caching~\citep{wimbauer2024cache}, Learning-to-Cache~\citep{ma2024learningtocache}, SmoothCache~\citep{liu2025smoothcache}, ProCache~\citep{cao2026procache}, $\Delta$-DiT~\citep{chen2024delta}, and SenCache~\citep{haghighi2026sencache} reuse work through fixed patterns, calibration, or online sensitivity. ACT~\citep{graves2016adaptive} and PonderNet~\citep{banino2021pondernet} learn input-dependent halting. Recent diffusion-LM systems---Fast-dLLM (arXiv:2505.22618), dKV-Cache (arXiv:2505.15781), and dLLM-Cache (arXiv:2506.06295)---adapt caching to bidirectional denoising; unlike causal KV caching, they are in-family comparisons for LLaDA. They remove reuse across iterations, whereas our method removes selected branch evaluations within an iteration. Direct composition on diffusion LMs remains unmeasured. Token methods such as SiTo~\citep{zhang2025sito} and DiffSparse~\citep{zhu2026diffsparse} instead target intra-layer redundancy.

Weight spectra have been connected to generalization~\citep{martin2021implicit}, spectral normalization~\citep{miyato2018spectral}, and attention Lipschitz bounds~\citep{kim2021lipschitz}. We use the same objects for a different purpose: allocating repeated inference work across residual branches.

\section{Spectral Layer Scheduling}
\label{sec:theory}

We first define the cached object and the score, then state the conditional perturbation argument that motivates the ranking. The argument is not a deployment certificate: the deployed schedule retains only its weight-computable term.

\subsection{Setup}

Let $\epsilon_\theta(z_t,t)$ be a diffusion model with deterministic update
\begin{equation*}
    z_{t-1}=\alpha_t z_t - \beta_t \epsilon_\theta(z_t,t), \qquad t=T,\ldots,1.
\end{equation*}
Index network evaluations by $s\in\{1,\ldots,T\}$, from the noisiest to the final denoising state. Write a residual unit as
\begin{equation*}
 h_{\ell,s}=h_{\ell-1,s}+F_\ell(h_{\ell-1,s};c_s),
\end{equation*}
where $c_s$ collects timestep and other conditioning. A lifetime $S_\ell$ evaluates the branch for $s\leq S_\ell$ and freezes it afterward. At refresh iteration $r_\ell(s)$, the executor stores
\begin{equation*}
 \Delta_{\ell,r_\ell(s)}=F_\ell(h_{\ell-1,r_\ell(s)};c_{r_\ell(s)}).
\end{equation*}
For $s>S_\ell$, it executes
\begin{equation*}
 h_{\ell,s}=h_{\ell-1,s}+\Delta_{\ell,r_\ell(s)}.
\end{equation*}
Thus $h_{\ell-1,s}$ remains current; only the residual-branch update is reused. Shape-changing or dependency-coupled operations are grouped into atomic, dependency-closed scheduled units. The objective is to reduce branch evaluations while controlling the deviation between scheduled output $\hat z_0$ and full output $z_0$.

\subsection{Spectral Concentration Ratio}

The score that drives all of our scheduling decisions is built directly from the singular values of pretrained weights.

\begin{definition}[Regularized Spectral Concentration Ratio]
Let $W\in\mathbb{R}^{m\times n}$ have singular values $\sigma_1\geq\cdots\geq\sigma_d$ with $d=\min(m,n)$, and let $k=\min(d,\max(1,\lfloor 0.1d\rfloor))$. For $\eta>0$,
\begin{equation}
\label{eq:scr-reg}
    \scr_\eta(W)=\log\left(
    \frac{\sum_{i=1}^{k}\sigma_i^2+\eta\|W\|_F^2}
    {\sum_{i=k+1}^{d}\sigma_i^2+\eta\|W\|_F^2}
    \right).
\end{equation}
\end{definition}

The regularizer $\eta\|W\|_F^2$ in Eq.~\eqref{eq:scr-reg} stabilizes the score for rank-deficient matrices and for the $d=1$ edge cases that arise with bias-like parameters; the regularized form is used in implementation, while the unregularized algebraic relation below supplies the analytic motivation.

\begin{lemma}[Spectral energy sandwich]
\label{lem:sandwich}
Let $0<E_k<\|W\|_F^2$ with $E_k=\sum_{i=1}^k\sigma_i^2$, $u=\log(E_k/(\|W\|_F^2-E_k))$, and $\rho(u)=e^u/(1+e^u)$. Then
\begin{equation*}
    \|W\|_F\sqrt{\frac{\rho(u)}{k}}
    \leq \|W\|_2
    \leq \|W\|_F\sqrt{\rho(u)} .
\end{equation*}
\end{lemma}
\begin{proof}
Since
\[
e^u=\frac{E_k}{\|W\|_F^2-E_k},
\]
we obtain
\[
\rho(u)=\frac{e^u}{1+e^u}
       =\frac{E_k}{\|W\|_F^2}.
\]

For the upper bound, since $\sigma_1^2=\|W\|_2^2\leq E_k$,
\[
\|W\|_2
\leq \sqrt{E_k}
= \|W\|_F\sqrt{\rho(u)}.
\]

For the lower bound, using $E_k\leq k\sigma_1^2$,
\[
\frac{E_k}{k}
\leq \sigma_1^2=\|W\|_2^2,
\]
hence
\[
\|W\|_2
\geq \sqrt{\frac{E_k}{k}}
= \|W\|_F\sqrt{\frac{\rho(u)}{k}}.
\]
\end{proof}

The sandwich is elementary but useful: leading-band concentration connects Frobenius energy to the spectral norms that enter local sensitivity bounds. It motivates the proxy; it does not establish that SCR alone orders freezing error.

\subsection{Block sensitivity from spectra}

We estimate how a pre-norm transformer block can amplify input perturbations. The quantities below follow standard softmax and feed-forward Lipschitz arguments~\citep{gao2017softmax}. Let $W_Q,W_K,W_V,W_O$ be attention projections; let $W_1,W_2$ be MLP projections; let $\kappa_{\mathrm{act}}$ bound the GELU derivative~\citep{hendrycks2016gelu}; and let $L_{N,1},L_{N,2}$ bound the normalization maps.

\begin{assumption}[Bounded local attention sensitivity]
\label{ass:attention}
For hidden states satisfying $\|x\|_F\leq B$, the attention map satisfies $\|A(x)\|_2\leq A_\star$, and its score-to-attention Jacobian is bounded by $S_\star$. A sequence-length factor $c_n$ converts token-level norms to flattened norms under the chosen operator convention.
\end{assumption}

\begin{lemma}[Pre-norm transformer block sensitivity]
\label{lem:lipschitz}
Under Assumption~\ref{ass:attention}, a pre-norm transformer block $h$ with Lipschitz normalization maps $N_1,N_2$ has $\mathrm{Lip}(h)\leq \Lambda_\ell$, where
\begin{equation*}
    \Lambda_\ell = (1+L_{N,1}L_{\mathrm{MHA}})(1+L_{N,2}L_{\mathrm{FFN}}),
\end{equation*}
\begin{equation*}
\begin{aligned}
L_{\mathrm{MHA}}
\leq\;&
\|W_O\|_2 \|W_V\|_2
\Bigg(
A_\star \\
&\quad
+ \frac{2 B c_n S_\star}{\sqrt{d_h}}
\|W_Q\|_2 \|W_K\|_2
\Bigg).
\end{aligned}
\end{equation*}
\begin{equation*}
    L_{\mathrm{FFN}} \leq \kappa_{\mathrm{act}}\|W_2\|_2\|W_1\|_2.
\end{equation*}
\end{lemma}
\begin{proof}
The attention derivative decomposes into a value-path term, bounded by $\|W_O\|_2\|W_V\|_2\|A(x)\|_2$, and a score-dependent attention-map term obtained by applying the product rule to the bilinear score map and multiplying by the local score-to-attention sensitivity. The feed-forward bound follows from submultiplicativity and the activation derivative bound. Residual composition then gives the product form.
\end{proof}

Substituting Lemma~\ref{lem:sandwich} into Lemma~\ref{lem:lipschitz} replaces each spectral norm by a computable SCR/Frobenius quantity. For scheduled matrix $W_j^{(\ell)}$, define
\begin{equation*}
    \gamma_j^{(\ell)}=\|W_j^{(\ell)}\|_F\sqrt{\rho(\scr_\eta(W_j^{(\ell)}))},
\end{equation*}
and substitute $\gamma_j^{(\ell)}$ for $\|W_j^{(\ell)}\|_2$. The resulting weight-only block score is $\score_\ell$. This route directly covers only pre-norm attention and MLP projections under the stated local assumptions. For AdaLN, cross-attention conditioning, convolutions, and U-shaped skips, the same statistic is a heuristic tested by experiment. SDXL kernels are unfolded to $W\in\mathbb{R}^{C_{\mathrm{out}}\times(C_{\mathrm{in}}K_HK_W)}$ before scoring. Table~\ref{tab:architecture-coverage} separates direct motivation from empirical transfer.

\begin{table*}[t]
\centering
\caption{Scope and executor semantics. ``Direct'' means the local argument covers the stated pre-norm branch under its assumptions; it is not a global certificate.}
\label{tab:architecture-coverage}

\small
\setlength{\tabcolsep}{4pt}

\begin{tabular}{
p{1.6cm}
p{2.5cm}
p{2.4cm}
p{1.7cm}
p{4.2cm}
}
\toprule
\textbf{Model} &
\textbf{Main block type} &
\textbf{Conditioning} &
\textbf{Theory coverage} &
\textbf{Cached update / current signal} \\
\midrule

LLaDA-8B &
transformer diffusion LM &
RMSNorm / timestep conditioning &
direct for pre-norm branches &
attention/MLP update; current residual stream \\

DiT-XL/2 &
diffusion transformer &
AdaLN / class conditioning &
heuristic for AdaLN &
attention/MLP update; current residual stream \\

U-ViT-L &
U-shaped ViT &
skip connections / patch features &
partial &
group update; current main stream and skip input \\

SDXL &
latent UNet &
conv, cross-attn, residual blocks &
heuristic &
valid residual/attention update; current residual and UNet skips \\

\bottomrule
\end{tabular}
\end{table*}

\subsection{Trajectory error from freezing}

Weight sensitivity is only one factor. A trajectory statement also needs cache-age drift and downstream amplification.

\begin{assumption}[Hidden-state drift and downstream amplification]
\label{ass:hidden}
Let the augmented branch input be $u_{\ell,s}=(h_{\ell-1,s},c_s)$. For each frozen iteration $s>S_\ell$,
\begin{equation*}
\|u_{\ell,s}-u_{\ell,r_\ell(s)}\|_2 \leq D_{\ell} (s-r_\ell(s)) V_{\max},
\end{equation*}
where $V_{\max}=\max_t\|z_t-z_{t-1}\|_2$. This is an assumed linear envelope, not a consequence of the weights. The downstream subnetwork mapping a branch perturbation to the denoiser output has local Lipschitz constant at most $\widehat P_{\ell,s}$.
\end{assumption}

\begin{theorem}[Conditional SCR-derived freezing bound]
\label{thm:main}
Under Assumptions~\ref{ass:attention} and~\ref{ass:hidden}, a static freezing schedule satisfies
\begin{equation}
\label{eq:main-bound}
\begin{aligned}
    \|\hat z_0-z_0\|_2
    \leq{}& C_TV_{\max}\sum_{s=1}^{T}\beta_s \\
    &\cdot\sum_{\ell:s>S_\ell}
    \widehat P_{\ell,s}\score_\ell D_\ell(s-r_\ell(s)),
\end{aligned}
\end{equation}
where $C_T$ collects the local diffusion-update amplification factors.
\end{theorem}
\begin{proof}[Proof sketch]
For a frozen unit, the relevant error is
$\|F_\ell(u_{\ell,s})-F_\ell(u_{\ell,r_\ell(s)})\|_2$, not the difference between complete hidden states. Branch Lipschitzness and Assumption~\ref{ass:hidden} bound it by $\score_\ell D_\ell(s-r_\ell(s))V_{\max}$ up to the constants absorbed into $\score_\ell$. A grouped residual unit satisfies the conservative bridge $\mathrm{Lip}(F_\ell)\leq \mathrm{Lip}(h_\ell)+1\leq\Lambda_\ell+1$. Multiplying by $\widehat P_{\ell,s}$, summing frozen branches, and unrolling the diffusion recurrence gives Eq.~\eqref{eq:main-bound}. Appendix~\ref{app:proof} states the scope of this argument.
\end{proof}

\begin{corollary}[Budgeted freezing]
\label{cor:safe}
If $\widehat P_{\ell,s}$ and $D_\ell$ are treated as layer-independent or slowly varying, freezing in increasing order of $\score_\ell$ minimizes the retained weight-only component at a fixed layer-step budget.
\end{corollary}
\begin{proof}
Under the stated approximation, all non-weight terms in Eq.~\eqref{eq:main-bound} are treated as layer-independent or slowly varying, so the dominant schedule-dependent term is ordered by $\score_\ell$.
\end{proof}

Theorem~\ref{thm:main} and Corollary~\ref{cor:safe} provide motivation, not a proof that the deployed schedule minimizes Eq.~\eqref{eq:main-bound}. The scheduler omits $D_\ell$ and $\widehat P_{\ell,s}$ because estimating them requires trajectories. Section~\ref{sec:experiments} tests the weight-only ordering, and Appendix~\ref{app:proof} audits empirical directional analogues of the omitted factors without using them for deployment.

\section{Method}
\label{sec:method}

The method turns the weight score into lifetimes, then compiles their repeated active/frozen patterns into executable graphs.

\subsection{Schedule construction}

Given a pretrained model, horizon $T$, minimum lifetime $S_{\min}$, and scale $\tau$, we compute $\score_\ell\in[0,1]$ offline and assign
\begin{equation}
\label{eq:schedule}
    S_\ell=\max\left\{S_{\min},\left\lceil T\min\left(1,\frac{\score_\ell}{\tau}\right)\right\rceil\right\}.
\end{equation}
High-scoring units remain active longer. Algorithm~\ref{alg:schedule} gives the procedure; Appendix~\ref{app:block-scores} gives the exact attention, MLP, cross-attention, convolutional, aggregation, and normalization equations. The schedule uses no activation statistic, validation prompt, input-specific router, or per-input search. The only quality--latency scalar is $\tau$.
\begin{algorithm}[t]
\caption{Spectral-Guided Schedule Construction}
\label{alg:schedule}
\begin{algorithmic}[1]
\STATE \textbf{Input:} pretrained model $\theta$, denoising iterations $T$, scale $\tau$, minimum lifetime $S_{\min}$
\FOR{each scheduled layer $\ell$}
    \FOR{each scheduled matrix or unfolded tensor $W_j^{(\ell)}$}
        \STATE Compute $\scr_\eta(W_j^{(\ell)})$ and $\|W_j^{(\ell)}\|_F$
        \STATE Compute $g(W_j^{(\ell)})$ using Eq.~\eqref{eq:app-matrix-score}
    \ENDFOR
    \STATE Compute raw block score $\tilde q_\ell$ using Eq.~\eqref{eq:app-block-score}
\ENDFOR
\STATE Normalize raw scores to $\score_\ell$ using Eq.~\eqref{eq:app-score-normalization}
\STATE Set $S_\ell$ using Eq.~\eqref{eq:schedule}
\STATE \textbf{Return:} static lifetimes $\{S_\ell\}_{\ell=1}^{L}$
\end{algorithmic}
\end{algorithm}

\subsection{Residual-update and jagged execution}

Heterogeneous lifetimes create a jagged sequence of active/frozen patterns. Each pattern is topologically valid: independently bypassable residual branches may freeze, while shape-changing or dependency-coupled operations form atomic groups. For every static pattern, we capture a CUDA Graph. During replay there is no Python branch per layer. Active units compute and cache $\Delta_{\ell,s}$; frozen units add $\Delta_{\ell,r_\ell(s)}$ to the current residual stream. Current skip features and conditioning still enter every valid group boundary. Algorithm~\ref{alg:jagged} states this data flow. Main experiments keep caches on GPU; pinned-host fallback is a deployment option, not part of reported latency.

Table~\ref{tab:kernel} is crucial for interpreting speedup. Relative to eager full inference, Python skipping gives $1.9\times$; a padded graph that still executes all $3200$ layer-steps gives $2.7\times$; and the jagged graph executes $1120$ steps and gives $3.0\times$. Thus graph capture accounts for most of the headline gain, while removing padded branch work contributes the remaining improvement. Because the denominator is eager execution, $3.0\times$ can exceed the $1/0.35\approx2.86\times$ ceiling implied by branch work alone; this is overhead removal, not superlinear compute savings. The headline is a complete-system comparison, not a $3.0\times$ causal gain from freezing.

\begin{algorithm}[t]
\caption{Jagged-Stream Execution}
\label{alg:jagged}
\begin{algorithmic}[1]
\STATE \textbf{Input:} current state, conditioning $c_s$, lifetimes $\{S_\ell\}$, caches $\{\Delta_\ell\}$
\STATE Form the dependency-closed active/frozen pattern for iteration $s$
\STATE Replay its captured CUDA Graph
\FOR{each scheduled residual unit $\ell$ in graph order}
  \IF{$s\leq S_\ell$}
    \STATE Compute $\Delta_{\ell,s}=F_\ell(h_{\ell-1,s};c_s)$ and refresh its cache
  \ENDIF
  \IF{$s>S_\ell$}
    \STATE Read cached $\Delta_{\ell,r_\ell(s)}$
  \ENDIF
  \STATE Set $h_{\ell,s}=h_{\ell-1,s}+\Delta_\ell$ using the current residual stream
\ENDFOR
\end{algorithmic}
\end{algorithm}

\begin{table}[t]
\centering
\caption{LLaDA execution ablation. The $3.0\times$ result combines graph execution and branch skipping relative to eager full inference.}
\label{tab:kernel}

\small
\setlength{\tabcolsep}{3pt}

\begin{tabular}{
p{2.6cm}
p{1.5cm}
p{1.5cm}
p{1.2cm}
}
\toprule
\textbf{Execution mode} &
\textbf{Layer-steps} &
\textbf{Latency} &
\textbf{Speedup} \\
\midrule

Full model, eager &
3200 &
$450\pm6$ ms &
1.0$\times$ \\

Python skipping &
1120 &
$238\pm7$ ms &
1.9$\times$ \\

Padded static graph &
3200 &
$166\pm5$ ms &
2.7$\times$ \\

Jagged-stream executor &
1120 &
$150\pm4$ ms &
3.0$\times$ \\

\bottomrule
\end{tabular}
\end{table}

\section{Experiments}
\label{sec:experiments}

We test (i) matched-budget ranking quality, (ii) language-task breadth and short-horizon robustness, (iii) transfer across architectures, and (iv) the costs of static versus input-aware execution.

\subsection{Setup}

We evaluate LLaDA-8B on Needle-in-Haystack retrieval, GSM8K~\citep{cobbe2021gsm8k}, CNN/DailyMail (CNN/DM) summarization~\citep{hermann2015teaching}, MATH-500, HumanEval, and open-ended continuation. DiT-XL/2~\citep{peebles2023scalable} and U-ViT-L~\citep{bao2023all} use ImageNet-256~\citep{deng2009imagenet}; SDXL~\citep{podell2024sdxl} uses COCO~\citep{lin2014microsoft}. We report FID~\citep{heusel2017gans} from 50,000 ImageNet samples and 30,000 COCO validation prompts. Latency is end-to-end, post-warmup, on H100 80GB. Values are mean $\pm$ standard deviation over 10 seeds where stochasticity is present.

Direct baselines use the same checkpoint, precision, batch, split, resolution, guidance, and hardware. Static rankings match the spectral layer-step budget. Short-horizon controls also hold the default deterministic sampler fixed. A dash means that the quantity was not measured for that model; we do not reuse a latency from another architecture. Appendix~\ref{app:protocol} gives full protocols and baseline status.
\subsection{Cross-architecture transfer}

Table~\ref{tab:cross-model} reports the selected operating points. The measured $2.8\times$--$3.0\times$ values compare the complete captured-graph system with eager full inference; they should not be interpreted as inverse layer-step ratios. Table~\ref{tab:fixed-tau} fixes $\tau=1.0$, $S_{\min}=0.1T$, and $k=\lfloor0.1d\rfloor$ without per-model tuning. The same rule retains small quality changes across all four models, while the theoretical status still differs by architecture as Table~\ref{tab:architecture-coverage} makes explicit.

\begin{table*}[t]
\centering
\caption{Cross-model operating points. Speedup is complete-system wall time versus eager full inference; quality change uses the same sampler and evaluation setting.}
\label{tab:cross-model}

\small
\setlength{\tabcolsep}{4pt}

\begin{tabular}{
p{1.6cm}
p{2.3cm}
p{2.1cm}
p{1.8cm}
p{2.0cm}
p{1.2cm}
p{3.0cm}
}
\toprule
\textbf{Model} &
\textbf{Architecture} &
\textbf{Task} &
\textbf{Full latency} &
\textbf{Scheduled latency} &
\textbf{Speedup} &
\textbf{Quality change} \\
\midrule

LLaDA-8B &
diffusion LM &
GSM8K / Needle &
$450\pm6$ ms &
$150\pm4$ ms &
3.0$\times$ &
$-0.2$ GSM8K / $-0.2$ Needle \\

DiT-XL/2 &
diffusion transformer &
ImageNet-256 &
$1.82\pm0.03$ s &
$0.61\pm0.01$ s &
3.0$\times$ &
$+0.08$ FID \\

U-ViT-L &
U-ViT &
ImageNet-256 &
$1.24\pm0.02$ s &
$0.45\pm0.01$ s &
2.8$\times$ &
$+0.07$ FID \\

SDXL &
latent UNet &
COCO &
$4.80\pm0.05$ s &
$1.60\pm0.03$ s &
3.0$\times$ &
$+0.25$ COCO FID \\

\bottomrule
\end{tabular}
\end{table*}

\begin{table*}[t]
\centering
\caption{Untuned transfer. Hyperparameters are fixed across models. Layer-step ratios and wall-clock speedups are distinct quantities.}
\label{tab:fixed-tau}

\small
\setlength{\tabcolsep}{4pt}

\begin{tabular}{lccccc}
\toprule
\textbf{Model} &
\textbf{Metric} &
\textbf{Layer-steps} &
\textbf{Speedup} &
\textbf{Full} &
\textbf{Scheduled} \\
\midrule
LLaDA-8B & GSM8K Acc. & 35\% & 3.0$\times$ & $68.3\pm0.5$ & $68.1\pm0.5$ \\
DiT-XL/2 & FID $\downarrow$ & 34\% & 3.0$\times$ & $2.27\pm0.03$ & $2.35\pm0.04$ \\
U-ViT-L & FID $\downarrow$ & 36\% & 2.8$\times$ & $3.08\pm0.04$ & $3.15\pm0.05$ \\
SDXL & COCO FID $\downarrow$ & 32\% & 3.0$\times$ & $7.84\pm0.07$ & $8.09\pm0.10$ \\
\bottomrule
\end{tabular}
\end{table*}

\subsection{Static ranking comparisons}

Table~\ref{tab:cross-ablation} isolates ranking quality by holding each model's layer-step budget fixed. SCR/Frobenius is best in every column. Raw SCR also beats Frobenius, spectral norm, stable rank, and the added Frobenius--stable-rank product. At the layer level, $\rho_{\mathrm{sp}}(\mathrm{SCR},\mathrm{stable\ rank})$ is $-0.31/-0.27/-0.24/-0.35$ for LLaDA/DiT/U-ViT/SDXL, indicating related but non-redundant orderings. Some margins, especially LLaDA, overlap the displayed uncertainty; without paired-test values we make no paired-significance claim.

\begin{table*}[t]
\centering
\caption{Static rankings at matched layer-step budgets. Larger raw scores receive longer lifetimes after orienting every baseline toward measured sensitivity.}
\label{tab:cross-ablation}
\scriptsize
\setlength{\tabcolsep}{4pt}
\begin{tabular}{lcccc}
\toprule
Ranking score & LLaDA GSM8K $\uparrow$ & DiT FID $\downarrow$ & U-ViT FID $\downarrow$ & SDXL COCO FID $\downarrow$ \\
\midrule
Random & $63.4\pm0.8$ & $3.21\pm0.08$ & $4.02\pm0.09$ & $9.45\pm0.16$ \\
Depth-only & $66.2\pm0.6$ & $2.87\pm0.06$ & $3.61\pm0.07$ & $8.88\pm0.13$ \\
Frobenius norm & $67.2\pm0.5$ & $2.74\pm0.06$ & $3.44\pm0.06$ & $8.64\pm0.12$ \\
Spectral norm & $66.8\pm0.6$ & $2.69\pm0.05$ & $3.39\pm0.06$ & $8.57\pm0.12$ \\
Stable rank & $67.5\pm0.5$ & $2.61\pm0.05$ & $3.31\pm0.06$ & $8.43\pm0.11$ \\
Frobenius $\times$ stable rank & $67.6\pm0.5$ & $2.55\pm0.05$ & $3.27\pm0.05$ & $8.34\pm0.11$ \\
Raw SCR & $67.9\pm0.5$ & $2.43\pm0.04$ & $3.22\pm0.05$ & $8.19\pm0.10$ \\
SCR/Frobenius $\score$ & $\mathbf{68.1\pm0.5}$ & $\mathbf{2.35\pm0.04}$ & $\mathbf{3.15\pm0.05}$ & $\mathbf{8.09\pm0.10}$ \\
\bottomrule
\end{tabular}
\end{table*}

\subsection{Language breadth}

Table~\ref{tab:language} expands LLaDA beyond the submitted retrieval, GSM8K, and summarization set. At the same $1120/3200=35\%$ layer-step budget, spectral scheduling retains more quality than depth and random rankings on maximum-context retrieval, MATH-500, HumanEval, and MAUVE. This supports the ranking across retrieval, reasoning, code, summarization, and open-ended generation, but it remains evidence from one diffusion-LM backbone. The latency column is the standard-length setting; no absolute max-context latency was provided, so we do not reuse it for the 4096-token Needle result.

\begin{table*}[t]
\centering
\caption{LLaDA-8B language evaluation. ``Needle'' is the original setting; ``Needle-4K'' uses the 4096-token limit. Standard latency does not represent the max-context run.}
\label{tab:language}
\scriptsize
\setlength{\tabcolsep}{2.2pt}
\begin{tabular}{lccrrrrrrr}
\toprule
Method & Steps & Std. latency & Needle & GSM8K & CNN/DM & Needle-4K & MATH & HumanEval & MAUVE \\
\midrule
Full & 3200 & $450\pm6$ ms & $100.0\pm0.0$ & $68.3\pm0.5$ & $42.1\pm0.2$ & $99.6\pm0.1$ & $28.6\pm0.6$ & $32.9\pm0.7$ & $.921\pm.006$ \\
Random & 1120 & $150\pm5$ ms & $81.2\pm1.0$ & $63.4\pm0.8$ & $38.2\pm0.4$ & $77.9\pm1.1$ & $22.1\pm0.8$ & $25.7\pm0.9$ & $.846\pm.010$ \\
Depth & 1120 & $150\pm5$ ms & $90.1\pm0.8$ & $66.2\pm0.6$ & $40.1\pm0.3$ & $87.8\pm0.9$ & $25.9\pm0.7$ & $29.8\pm0.8$ & $.884\pm.008$ \\
\textbf{Spectral} & 1120 & $\mathbf{150\pm4}$ ms & $\mathbf{99.8\pm0.1}$ & $\mathbf{68.1\pm0.5}$ & $\mathbf{41.9\pm0.2}$ & $\mathbf{99.2\pm0.2}$ & $\mathbf{28.2\pm0.6}$ & $\mathbf{32.4\pm0.7}$ & $\mathbf{.916\pm.006}$ \\
Oracle replay & $930\pm35$ & $125\pm4$ ms & $100.0\pm0.0$ & $68.3\pm0.5$ & $42.1\pm0.2$ & -- & -- & -- & -- \\
\bottomrule
\end{tabular}
\end{table*}

\subsection{Short horizons and input-aware control}

Table~\ref{tab:short-horizon} holds the default deterministic sampler and layer-step ratio fixed while reducing $T$. Spectral remains the best static ranking at every tested horizon. The gap to Full grows at $T=10$, as each retained evaluation matters more, yet the advantage over depth and random persists. This control was run on LLaDA and DiT; no U-ViT or SDXL sweep is reported.

\begin{table*}[t]
\centering
\caption{Matched-ratio short-horizon control. DiT entries are $\Delta$FID from Full at the same $T$; lower is better.}
\label{tab:short-horizon}
\scriptsize
\setlength{\tabcolsep}{3.5pt}
\begin{tabular}{c|rrrr|rrr}
\toprule
$T$ & LLaDA Full & Spectral & Depth & Random & DiT Spectral & DiT Depth & DiT Random \\
\midrule
100 & $68.3\pm.5$ & $\mathbf{68.1\pm.5}$ & $66.2\pm.6$ & $63.4\pm.8$ & $\mathbf{+.08\pm.03}$ & $+.60\pm.04$ & $+.94\pm.05$ \\
50  & $67.6\pm.6$ & $\mathbf{67.3\pm.6}$ & $64.9\pm.7$ & $61.9\pm.8$ & $\mathbf{+.11\pm.03}$ & $+.67\pm.04$ & $+.98\pm.05$ \\
30  & $66.4\pm.6$ & $\mathbf{66.0\pm.6}$ & $62.8\pm.7$ & $59.4\pm.9$ & $\mathbf{+.16\pm.04}$ & $+.82\pm.05$ & $+1.17\pm.05$ \\
20  & $64.9\pm.7$ & $\mathbf{64.3\pm.7}$ & $60.7\pm.8$ & $56.8\pm.9$ & $\mathbf{+.23\pm.04}$ & $+1.03\pm.05$ & $+1.39\pm.06$ \\
10  & $59.2\pm.9$ & $\mathbf{58.1\pm.9}$ & $53.6\pm1.0$ & $49.7\pm1.1$ & $\mathbf{+.39\pm.05}$ & $+1.36\pm.06$ & $+1.76\pm.06$ \\
\bottomrule
\end{tabular}
\end{table*}

The causal activation-drift baseline in Table~\ref{tab:input-aware} activates, at step $s$, the same number of units as the static schedule but chooses those with the largest one-step-lagged normalized drift. It is slightly better numerically on quality, within displayed uncertainty, and consistently slower because it performs online reductions, routing, and more variable graph selection. We therefore claim lower control overhead and deterministic deployment, not quality dominance over input-aware policies. The oracle is a replay-only, full-trajectory upper reference; Appendix~\ref{app:protocol} defines both policies.

\begin{table*}[t]
\centering
\caption{Static versus input-aware scheduling at matched per-step active counts. Quality differences are within displayed uncertainty.}
\label{tab:input-aware}
\scriptsize
\setlength{\tabcolsep}{4pt}
\begin{tabular}{llccc}
\toprule
Method & Signal & LLaDA ms / GSM8K $\uparrow$ & DiT s / FID $\downarrow$ & SDXL s / FID $\downarrow$ \\
\midrule
Spectral static & weights, offline & $\mathbf{150\pm4}/68.1\pm.5$ & $\mathbf{.61\pm.01}/2.35\pm.04$ & $\mathbf{1.60\pm.03}/8.09\pm.10$ \\
Activation drift & per-input drift & $171\pm6/\mathbf{68.2\pm.5}$ & $.69\pm.02/\mathbf{2.32\pm.04}$ & $1.78\pm.04/\mathbf{8.02\pm.10}$ \\
Oracle replay & full-output agreement & $125\pm4/68.3\pm.5$ & -- & -- \\
\bottomrule
\end{tabular}
\end{table*}

\subsection{Class-conditional image generation}

Table~\ref{tab:vision_class} reports DiT-XL/2 and U-ViT-L on ImageNet-256. Spectral preserves FID better than random or depth-only scheduling at matched compute and composes with DeepCache. Latency columns are model-specific: U-ViT latencies unavailable in the supplied measurements are marked ``--'' rather than inheriting DiT timings.

\begin{table*}[t]
\centering
\caption{ImageNet-256 results. Separate latency columns prevent DiT timings from being read as U-ViT timings.}
\label{tab:vision_class}

\scriptsize
\setlength{\tabcolsep}{3.5pt}
\renewcommand{\arraystretch}{1.1}

\begin{tabular}{
>{\raggedright\arraybackslash}p{3.5cm}
>{\centering\arraybackslash}p{2.2cm}
>{\centering\arraybackslash}p{2.1cm}
>{\centering\arraybackslash}p{2.2cm}
>{\centering\arraybackslash}p{2.1cm}
}
\toprule
\textbf{Method} &
\textbf{DiT latency} &
\textbf{DiT FID $\downarrow$} &
\textbf{U-ViT latency} &
\textbf{U-ViT FID $\downarrow$} \\
\midrule

Full Model &
$1.82\pm0.03$ s &
$2.27\pm0.03$ &
$1.24\pm0.02$ s &
$3.08\pm0.04$ \\

DDIM, 50 steps &
$0.91\pm0.02$ s &
$3.14\pm0.08$ &
-- &
$3.91\pm0.09$ \\

DPM-Solver++ &
$0.38\pm0.01$ s &
$2.51\pm0.05$ &
-- &
$3.42\pm0.07$ \\

DeepCache &
$0.42\pm0.01$ s &
$2.68\pm0.06$ &
-- &
$3.34\pm0.06$ \\

$\Delta$-DiT &
$1.13\pm0.02$ s &
$2.41\pm0.05$ &
-- &
-- \\

Random static &
$0.61\pm0.01$ s &
$3.21\pm0.08$ &
-- &
$4.02\pm0.09$ \\

Depth static &
$0.61\pm0.01$ s &
$2.87\pm0.06$ &
-- &
$3.61\pm0.07$ \\

\midrule

\textbf{Spectral static} &
$\mathbf{0.61\pm0.01}$ s &
$\mathbf{2.35\pm0.04}$ &
$\mathbf{0.45\pm0.01}$ s &
$\mathbf{3.15\pm0.05}$ \\

Spectral + DeepCache &
$0.25\pm0.01$ s &
$2.48\pm0.05$ &
-- &
$3.39\pm0.06$ \\

\bottomrule
\end{tabular}
\end{table*}

\subsection{Text-to-image generation}

SDXL is the clearest out-of-scope architecture for the theory. Table~\ref{tab:sdxl} is therefore an empirical transfer result. Spectral has lower FID degradation than depth scheduling and the reported caching baselines at comparable latency, and its composition with DPM-Solver++ reaches $8\times$ relative to eager full SDXL. We do not treat this result as certification of cross-attention or convolutional blocks.

\subsection{What the schedule looks like}

Figure~\ref{fig:heatmaps} visualizes representative lifetimes. It is not a compute-budget audit: displayed groups omit multiplicities, so their simple average cannot recover the 32--36\% layer-step ratios. Budgets are computed from the exact unbinned sum $\sum_\ell S_\ell/(LT)$; Appendix Table~\ref{tab:schedule-groups} reports only the display coordinates and makes this limitation explicit.

\begin{figure}[t]
\centering
\begin{tikzpicture}[x=0.20cm,y=0.20cm,every node/.style={font=\footnotesize}]
\foreach \l in {1,...,16} {\draw[fill=gray!12,draw=white] (\l,0) rectangle +(1,16);}
\foreach \l/\h in {1/4,2/5,3/5,4/6,5/6,6/7,7/8,8/8,9/9,10/10,11/11,12/12,13/13,14/14,15/15,16/16} {\draw[fill=blue!65,draw=white] (\l,0) rectangle +(1,\h);}
\draw (1,0) rectangle (17,16);
\node[rotate=90] at (-0.7,8) {layer group};
\node at (9,-1.2) {denoising iteration};
\end{tikzpicture}
\caption{Representative LLaDA lifetime map. Colored groups are active; gray groups reuse cached branch updates. Full four-model maps appear in Figure~\ref{fig:heatmaps-full}.}
\label{fig:heatmaps}
\end{figure}
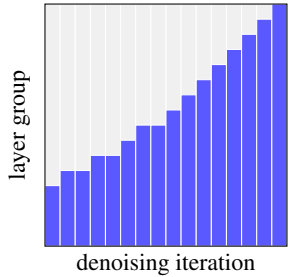

\subsection{Does the score predict freeze sensitivity?}

Figure~\ref{fig:score_scatter} compares $\score_\ell$ with one-unit freeze error. Spearman correlations are $0.66$--$0.76$: useful, but far from perfect. Table~\ref{tab:depth} controls for depth, block type, and Frobenius norm, and retains positive SCR coefficients. Appendix~\ref{app:proof} separately measures directional analogues of drift and downstream amplification under a schedule-independent probe. Their weak correlation with $\score_\ell$ and moderate dispersion explain why calibration-dependent reordering changes quality only marginally; they do not convert the theorem into a certificate.

\subsection{Quality--latency Pareto curves}

Figure~\ref{fig:pareto} sweeps $\tau$. Each model has a low-degradation region followed by sharper failure, but the knee is architecture-dependent and appears earliest for SDXL. Table~\ref{tab:fixed-tau} shows that $\tau=1$ is a usable shared operating point, not that one value is universally optimal. Appendix~\ref{app:failures} records the observed aggressive-schedule failures.

\section{Discussion and Conclusion}
\label{sec:discussion}

The stable result is the matched-budget ordering: SCR/Frobenius is the strongest tested static rule across four architectures and across a broader set of LLaDA behaviors. Short-horizon controls show that this is not only a consequence of a long redundant tail. The input-aware baseline recovers a small numerical quality gain but pays online routing and graph-pattern costs, which defines the intended deployment tradeoff.

The systems result requires separate accounting. Captured padded execution already removes most eager overhead; the jagged schedule then avoids inactive branch work. Reporting only the final $3.0\times$ would conflate these effects, while reporting only the $35\%$ layer-step ratio would ignore hardware overhead. Table~\ref{tab:kernel} gives both. The cost is peak memory: cached branch updates and graph workspace add $10$--$22\%$ across the evaluated models (Table~\ref{tab:memory}).

\begin{table}[H]
\centering
\caption{Peak H100 memory in GiB. Variation is at most $\pm0.1$ GiB.}
\label{tab:memory}
\scriptsize
\setlength{\tabcolsep}{2.2pt}
\begin{tabular}{lrrrr}
\toprule
Model / batch & Full & Sched. & Cache & Increase \\
\midrule
LLaDA / 1 & 18.9 & 22.2 & 3.0 & $17\%$ \\
DiT / 16 & 18.7 & 20.8 & 1.8 & $11\%$ \\
U-ViT / 16 & 14.9 & 16.4 & 1.3 & $10\%$ \\
SDXL / 8 & 39.8 & 48.5 & 8.2 & $22\%$ \\
\bottomrule
\end{tabular}
\end{table}

Peak allocation exceeds the cache itself by $0.2$--$0.5$ GiB because graph capture needs workspace, buffers and metadata, and because the caching allocator fragments under jagged replay. Every main run fits on H100 80GB and keeps the cache on device; reported latency never uses host offload. This matters operationally: the executor exchanges a predictable memory increment for a schedule with no runtime search. A smaller-memory deployment must instead choose between recomputation and transfer, so its wall-clock gain may differ from ours.

The post-hoc audit further separates a useful ranking from a guarantee. Across 128 full trajectories per model, the measured non-weight factor has coefficient of variation $.26$--$.41$ and Spearman correlation $-.04$--$.12$ with $\score_\ell$. Reordering by the full calibration-dependent product changes GSM8K by $+0.1$ and FID by only $-0.02$ to $-0.05$ (Appendix Table~\ref{tab:bound-audit}). These small changes support using the weight-only term when deterministic deployment is the priority; they do not show that drift and amplification are constant, nor that the static schedule is optimal.

In summary, pretrained spectral structure is a practical signal for deterministic resource allocation. The evidence supports a weight-only scheduler whose behavior transfers empirically, plus an executor that makes that schedule useful on hardware. Dynamic routing remains a sensible choice when a small numerical quality gain justifies online reductions, pattern variability, and a more complex serving path. Our result occupies the complementary regime: one schedule, known memory, no calibration prompts, and reproducible execution. It does not support a universal spectral certificate or a claim that static control dominates adaptive policies.

\section*{Limitations}
Only one evaluated backbone is a diffusion language model. The added tasks broaden behavior coverage, not model-family coverage. The short-horizon sweep is available only for LLaDA and DiT; U-ViT and SDXL remain open. The theory directly motivates pre-norm attention/MLP branches under local bounded-sensitivity and linear cache-age-drift assumptions. It does not certify AdaLN, cross-attention, convolutions, or U-shaped skips, and the measured directional factors are diagnostics rather than Lipschitz bounds.

A static schedule cannot adapt to unusually hard inputs. The activation-drift comparison suggests that dynamic control can recover small quality gains, although it incurs overhead. Several ablation margins overlap the displayed uncertainty, and paired-seed test statistics are not available, so we avoid significance claims for those margins. Recent diffusion-LM caches are discussed but not directly reproduced or composed with our LLaDA executor. We catalog failure categories but do not provide paired qualitative generations. Peak memory rises by $10$--$22\%$; all main runs fit on H100 80GB without host offload, but smaller devices may lose speed to transfers. Scores should be recomputed after aggressive quantization. Finally, Appendix~\ref{app:training} is preliminary; training-time freezing is outside the contribution.

\section*{Ethics Statement}
The method reduces inference cost and energy use for generative models. Because it is a general-purpose accelerator, it can also accelerate harmful uses of generative models, but it does not expand model capabilities, require new data collection, or introduce new human-subject risks.

\bibliography{references}

\clearpage
\appendix

\section{Tables and Figures}
\label{sec:tab_fig}

\begin{table}[t]
\centering
\caption{SDXL COCO text-to-image validation. Bold indicates the best standalone static schedule; composed methods are not bolded.}
\label{tab:sdxl}
\scriptsize
\begin{tabular}{@{}lcc@{}}
\toprule
Method & Latency & COCO FID $\downarrow$ \\
\midrule
Full SDXL & $4.80\pm0.05$ s & $7.84\pm0.07$ \\
DPM-Solver++ & $1.92\pm0.04$ s & $8.21\pm0.11$ \\
DeepCache & $2.20\pm0.04$ s & $8.56\pm0.12$ \\
Block Caching & $2.00\pm0.04$ s & $8.42\pm0.11$ \\
Depth static & $1.60\pm0.03$ s & $8.88\pm0.13$ \\
\midrule
\textbf{Spectral static} & $\mathbf{1.60\pm0.03}$ s & $\mathbf{8.09\pm0.10}$ \\
Spectral + DPM-Solver++ & $0.60\pm0.02$ s & $8.45\pm0.12$ \\
\bottomrule
\end{tabular}
\end{table}

\begin{table*}[t]
\centering
\caption{Depth-controlled regression of measured freeze error on SCR, layer depth, block type, and Frobenius norm.}
\label{tab:depth}
\small
\begin{tabular}{lccccc}
\toprule
Model & $\beta_{\mathrm{SCR}}$ & Std. err. & $p$-value & $R^2_{\mathrm{depth}}$ & $R^2_{\mathrm{full}}$ \\
\midrule
LLaDA-8B & 0.42 & 0.08 & $<10^{-3}$ & 0.31 & 0.49 \\
DiT-XL/2 & 0.36 & 0.07 & $<10^{-3}$ & 0.28 & 0.42 \\
U-ViT-L & 0.31 & 0.08 & $4.0\times10^{-3}$ & 0.25 & 0.36 \\
SDXL & 0.39 & 0.09 & $2.0\times10^{-3}$ & 0.22 & 0.38 \\
\bottomrule
\end{tabular}
\end{table*}

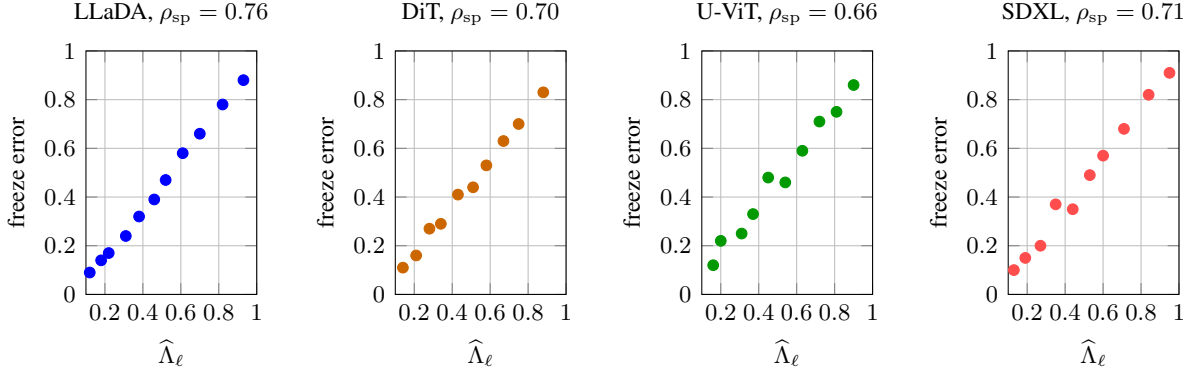
\begin{figure*}[t]
\centering
\begin{minipage}[t]{0.24\textwidth}
\centering
\begin{tikzpicture}
\begin{axis}[aclplot,title={LLaDA, $\rho_{\mathrm{sp}}=0.76$},xlabel={$\score_\ell$},ylabel={freeze error},xmin=0.1,xmax=1.0,ymin=0.0,ymax=1.0]
\addplot[only marks,mark=*,blue] coordinates {(0.12,0.09)(0.18,0.14)(0.22,0.17)(0.31,0.24)(0.38,0.32)(0.46,0.39)(0.52,0.47)(0.61,0.58)(0.70,0.66)(0.82,0.78)(0.93,0.88)};
\end{axis}
\end{tikzpicture}
\end{minipage}\hfill
\begin{minipage}[t]{0.24\textwidth}
\centering
\begin{tikzpicture}
\begin{axis}[aclplot,title={DiT, $\rho_{\mathrm{sp}}=0.70$},xlabel={$\score_\ell$},ylabel={freeze error},xmin=0.1,xmax=1.0,ymin=0.0,ymax=1.0]
\addplot[only marks,mark=*,orange!80!black] coordinates {(0.14,0.11)(0.21,0.16)(0.28,0.27)(0.34,0.29)(0.43,0.41)(0.51,0.44)(0.58,0.53)(0.67,0.63)(0.75,0.70)(0.88,0.83)};
\end{axis}
\end{tikzpicture}
\end{minipage}\hfill
\begin{minipage}[t]{0.24\textwidth}
\centering
\begin{tikzpicture}
\begin{axis}[aclplot,title={U-ViT, $\rho_{\mathrm{sp}}=0.66$},xlabel={$\score_\ell$},ylabel={freeze error},xmin=0.1,xmax=1.0,ymin=0.0,ymax=1.0]
\addplot[only marks,mark=*,green!60!black] coordinates {(0.16,0.12)(0.20,0.22)(0.31,0.25)(0.37,0.33)(0.45,0.48)(0.54,0.46)(0.63,0.59)(0.72,0.71)(0.81,0.75)(0.90,0.86)};
\end{axis}
\end{tikzpicture}
\end{minipage}\hfill
\begin{minipage}[t]{0.24\textwidth}
\centering
\begin{tikzpicture}
\begin{axis}[aclplot,title={SDXL, $\rho_{\mathrm{sp}}=0.71$},xlabel={$\score_\ell$},ylabel={freeze error},xmin=0.1,xmax=1.0,ymin=0.0,ymax=1.0]
\addplot[only marks,mark=*,red!70] coordinates {(0.13,0.10)(0.19,0.15)(0.27,0.20)(0.35,0.37)(0.44,0.35)(0.53,0.49)(0.60,0.57)(0.71,0.68)(0.84,0.82)(0.95,0.91)};
\end{axis}
\end{tikzpicture}
\end{minipage}
\caption{Measured layer-freeze error versus the static spectral score. Each point is a scheduled layer or block group from the freeze-sensitivity sweep that underlies Table~\ref{tab:score-corr}; Spearman correlations appear in the panel titles.}
\label{fig:score_scatter}
\end{figure*}

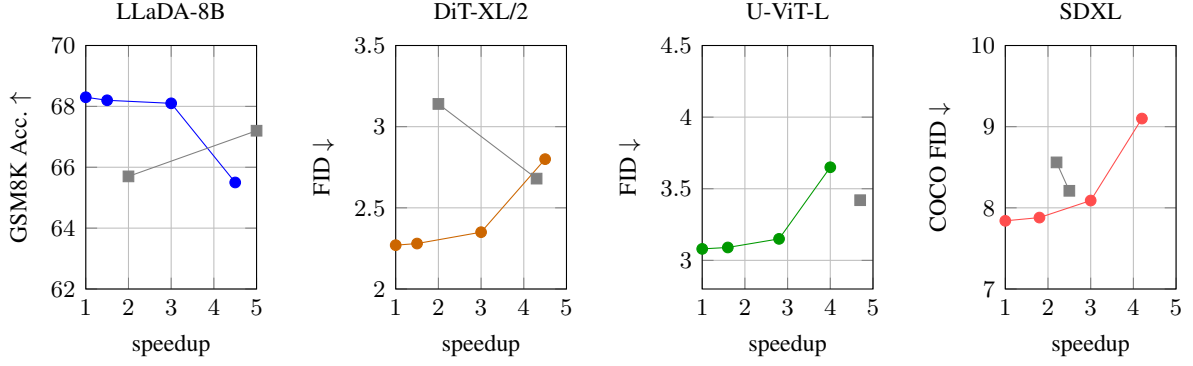
\begin{figure*}[t]
\centering
\begin{minipage}[t]{0.24\textwidth}
\centering
\begin{tikzpicture}
\begin{axis}[aclplot,title={LLaDA-8B},xlabel={speedup},ylabel={GSM8K Acc. $\uparrow$},xmin=1,xmax=5,ymin=62,ymax=70]
\addplot[color=blue, mark=*] coordinates {(1.0,68.3) (1.5,68.2) (3.0,68.1) (4.5,65.5)};
\addplot[color=gray, mark=square*] coordinates {(2.0,65.7) (5.0,67.2)};
\end{axis}
\end{tikzpicture}
\end{minipage}\hfill
\begin{minipage}[t]{0.24\textwidth}
\centering
\begin{tikzpicture}
\begin{axis}[aclplot,title={DiT-XL/2},xlabel={speedup},ylabel={FID $\downarrow$},xmin=1,xmax=5,ymin=2.0,ymax=3.5]
\addplot[color=orange!80!black, mark=*] coordinates {(1.0,2.27) (1.5,2.28) (3.0,2.35) (4.5,2.80)};
\addplot[color=gray, mark=square*] coordinates {(2.0,3.14) (4.3,2.68)};
\end{axis}
\end{tikzpicture}
\end{minipage}\hfill
\begin{minipage}[t]{0.24\textwidth}
\centering
\begin{tikzpicture}
\begin{axis}[aclplot,title={U-ViT-L},xlabel={speedup},ylabel={FID $\downarrow$},xmin=1,xmax=5,ymin=2.8,ymax=4.5]
\addplot[color=green!60!black, mark=*] coordinates {(1.0,3.08) (1.6,3.09) (2.8,3.15) (4.0,3.65)};
\addplot[color=gray, mark=square*] coordinates {(4.7,3.42)};
\end{axis}
\end{tikzpicture}
\end{minipage}\hfill
\begin{minipage}[t]{0.24\textwidth}
\centering
\begin{tikzpicture}
\begin{axis}[aclplot,title={SDXL},xlabel={speedup},ylabel={COCO FID $\downarrow$},xmin=1,xmax=5,ymin=7.0,ymax=10.0]
\addplot[color=red!70, mark=*] coordinates {(1.0,7.84) (1.8,7.88) (3.0,8.09) (4.2,9.10)};
\addplot[color=gray, mark=square*] coordinates {(2.5,8.21) (2.2,8.56)};
\end{axis}
\end{tikzpicture}
\end{minipage}
\caption{Quality--latency tradeoff across architectures. Circles are measured Spectral-Guided runs at different $\tau$ values; squares are directly reproduced baselines at their default operating points. Coordinates appear in Table~\ref{tab:pareto-coordinates}.}
\label{fig:pareto}
\end{figure*}

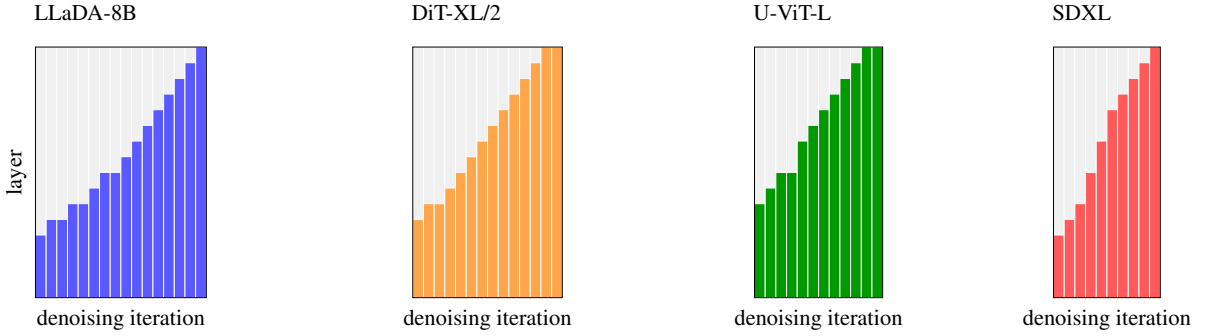
\begin{figure*}[t]
\centering
\resizebox{\textwidth}{!}{%
\begin{tikzpicture}[x=0.15cm,y=0.22cm,every node/.style={font=\footnotesize}]
\node[anchor=west] at (0,18.2) {LLaDA-8B};
\foreach \l in {1,...,16} {\draw[fill=gray!12,draw=white] (\l,0) rectangle +(1,16);}
\foreach \l/\h in {1/4,2/5,3/5,4/6,5/6,6/7,7/8,8/8,9/9,10/10,11/11,12/12,13/13,14/14,15/15,16/16} {\draw[fill=blue!65,draw=white] (\l,0) rectangle +(1,\h);}
\draw (1,0) rectangle (17,16);
\node[rotate=90] at (-0.8,8) {layer};
\node at (9,-1.4) {denoising iteration};
\begin{scope}[xshift=5.3cm]
\node[anchor=west] at (0,18.2) {DiT-XL/2};
\foreach \l in {1,...,14} {\draw[fill=gray!12,draw=white] (\l,0) rectangle +(1,16);}
\foreach \l/\h in {1/5,2/6,3/6,4/7,5/8,6/9,7/10,8/11,9/12,10/13,11/14,12/15,13/16,14/16} {\draw[fill=orange!70,draw=white] (\l,0) rectangle +(1,\h);}
\draw (1,0) rectangle (15,16);
\node at (8,-1.4) {denoising iteration};
\end{scope}
\begin{scope}[xshift=10.1cm]
\node[anchor=west] at (0,18.2) {U-ViT-L};
\foreach \l in {1,...,12} {\draw[fill=gray!12,draw=white] (\l,0) rectangle +(1,16);}
\foreach \l/\h in {1/6,2/7,3/8,4/8,5/10,6/11,7/12,8/13,9/14,10/15,11/16,12/16} {\draw[fill=green!60!black,draw=white] (\l,0) rectangle +(1,\h);}
\draw (1,0) rectangle (13,16);
\node at (7,-1.4) {denoising iteration};
\end{scope}
\begin{scope}[xshift=14.3cm]
\node[anchor=west] at (0,18.2) {SDXL};
\foreach \l in {1,...,10} {\draw[fill=gray!12,draw=white] (\l,0) rectangle +(1,16);}
\foreach \l/\h in {1/4,2/5,3/6,4/8,5/10,6/12,7/13,8/14,9/15,10/16} {\draw[fill=red!65,draw=white] (\l,0) rectangle +(1,\h);}
\draw (1,0) rectangle (11,16);
\node at (6,-1.4) {denoising iteration};
\end{scope}
\end{tikzpicture}
}
\caption{Full displayed lifetime maps. These visual groups omit multiplicities and must not be averaged to reconstruct layer-step budgets.}
\label{fig:heatmaps-full}
\end{figure*}

\section{Protocol Details}
\label{app:protocol}

All experiments used NVIDIA H100 80GB GPUs. Schedule construction, sweeps, and final evaluations consumed approximately 150 GPU hours. Models and datasets were used under their academic open-source licenses.

Table~\ref{tab:protocol} summarizes evaluation, Table~\ref{tab:baseline-config} lists sampler and budget settings, and Table~\ref{tab:baseline-status} separates matched results from literature-only positioning.

\begin{table*}[t]
\centering
\caption{Evaluation protocol. Within each task, methods share prompts, generation lengths, denoising settings, and metric implementations.}
\label{tab:protocol}
\scriptsize
\setlength{\tabcolsep}{3pt}
\renewcommand{\arraystretch}{1.12}
\begin{tabularx}{\textwidth}{@{}l X X c c c c X@{}}
\toprule
Model & Dataset & Samples & Resolution & Precision & Batch & Hardware & Primary metric \\
\midrule
LLaDA-8B & Needle / GSM8K / CNN/DM / MATH-500 / HumanEval / WikiText-103 & full official splits; 500 MATH problems; 164 code tasks; held-out continuation prompts & text & bf16 & 1 & H100 80GB & accuracy / exact match / pass@1 / ROUGE-L / MAUVE \\
DiT-XL/2 & ImageNet-256 & 50k samples & $256^2$ & bf16 & 16 & H100 80GB & FID \\
U-ViT-L & ImageNet-256 & 50k samples & $256^2$ & bf16 & 16 & H100 80GB & FID \\
SDXL & COCO validation prompts & 30k samples & $1024^2$ & fp16 & 8 & H100 80GB & COCO FID \\
\bottomrule
\end{tabularx}
\end{table*}

Needle-4K uses the checkpoint's maximum supported 4096-token context. MATH-500 is zero-shot and scores the final answer by exact match. HumanEval uses one deterministic completion per task (temperature 0, at most 512 generated tokens) and reports pass@1. MAUVE compares 128-token continuations from held-out WikiText-103 prompts with corresponding human references. The absolute latency for Needle-4K was not reported; Table~\ref{tab:language} therefore labels its latency as the standard-length setting.

\begin{table*}[h]
\centering
\caption{Baseline and sampler settings. Static rows match the spectral layer-step budget. LLaDA solver rows are retained as exploratory results but excluded from matched claims because their masked-denoising adaptation is insufficiently specified.}
\label{tab:baseline-config}
\scriptsize
\setlength{\tabcolsep}{3pt}
\begin{tabularx}{\textwidth}{l l c X c X}
\toprule
Family & Method & Steps & Sampler / order & Guidance & Budget / status \\
\midrule
LLaDA-8B & Full & 100 & default deterministic & n/a & 3200 layer-steps; matched \\
LLaDA-8B & DDIM / reduced & 50 & DDIM & n/a & full layers; exploratory \\
LLaDA-8B & DPM-Solver++ & 20 & second order & n/a & full layers; exploratory \\
LLaDA-8B & Static & 100 & default deterministic & n/a & 1120 layer-steps; matched \\
DiT-XL/2 & Full / static & 100 & default & cfg 1.5 & full / 34\%; matched \\
DiT-XL/2 & DDIM / DPM-Solver++ & 50 / 20 & DDIM / second order & cfg 1.5 & full layers; matched \\
DiT-XL/2 & DeepCache / $\Delta$-DiT & 100 & default & cfg 1.5 & method default; matched \\
U-ViT-L & Full / static & 100 & default & cfg 1.5 & full / 36\%; matched \\
SDXL & Full / static & 50 & default & cfg 7.5 & full / 32\%; matched \\
SDXL & DPM-Solver++ & 20 & second order & cfg 7.5 & full UNet; matched \\
\bottomrule
\end{tabularx}
\end{table*}

The two excluded LLaDA solver measurements are shown in Table~\ref{tab:language-solvers} so the numerical record is preserved without using them to support the matched-compute claim.

\begin{table*}[h]
\centering
\caption{Exploratory LLaDA reduced-step results retained for completeness but excluded from matched claims.}
\label{tab:language-solvers}
\scriptsize
\begin{tabular}{lrrrr}
\toprule
Method & Latency & Needle & GSM8K & CNN/DM \\
\midrule
DDIM / reduced & $225\pm4$ ms & $94.2\pm.7$ & $65.7\pm.6$ & $40.8\pm.3$ \\
DPM-Solver++ & $90\pm3$ ms & $96.1\pm.6$ & $67.2\pm.5$ & $41.4\pm.2$ \\
\bottomrule
\end{tabular}
\end{table*}

\begin{table*}[h]
\centering
\caption{Baseline status. Literature rows are positioning only and are not mixed into matched-compute claims.}
\label{tab:baseline-status}
\scriptsize
\setlength{\tabcolsep}{3pt}
\begin{tabularx}{\textwidth}{@{}p{3.4cm}p{2.5cm}cX@{}}
\toprule
Method & Status & Used in matched tables & Notes \\
\midrule
DDIM / DPM-Solver++ (vision) & direct & yes & same checkpoints, precision, resolution, and batch size \\
DDIM / DPM-Solver++ (LLaDA) & direct, exploratory & no & masked-denoising adaptation not sufficiently documented \\
DeepCache & direct & yes & compatible image models with method-default cache interval \\
$\Delta$-DiT & direct for DiT & yes & unavailable for U-ViT row \\
Random / depth static & direct & yes & matched layer-step budget to Spectral-Guided \\
SmoothCache & literature positioning & no & different public settings \\
ProCache & literature positioning & no & different public settings \\
Learning-to-Cache & literature positioning & no & different public settings \\
Fast-dLLM / dKV-Cache / dLLM-Cache & literature positioning & no & diffusion-LM in-family caching; not reproduced \\
\bottomrule
\end{tabularx}
\end{table*}

\subsection{Dynamic baselines}

\paragraph{Causal activation drift.}
After step $s-1$, the input-aware policy records
\begin{equation*}
d_{\ell,s-1}=
\frac{\|h_{\ell-1,s-1}-h_{\ell-1,s-2}\|_2}
{\|h_{\ell-1,s-1}\|_2+10^{-8}}.
\end{equation*}
At step $s$, it activates exactly the $m_s$ units with largest lagged drift, where $m_s$ equals the static schedule's active count at that step; ties favor the lower unit index. The first step is full. Consequently the per-step active-count profile and total layer-step budget match Spectral, while unit identity is input-dependent. The policy uses the same graph library; unseen patterns fall back to eager execution. Reported latency includes norm reductions, routing, graph lookup, and eager fallback.

\paragraph{Oracle replay.}
This non-deployable reference first records the full denoiser trajectory for each input. At step $s$, it greedily grows a cumulative frozen set $\mathcal F_s$. For every remaining candidate $j$, it evaluates $\mathcal F_s\cup\{j\}$ and measures
\begin{equation*}
\delta(\mathcal F_s\cup\{j\},s)=
\frac{\|\hat\epsilon_\theta^{(\mathcal F_s\cup\{j\})}(z_s,s)-
\hat\epsilon_\theta(z_s,s)\|_2}
{\|\hat\epsilon_\theta(z_s,s)\|_2+10^{-8}}.
\end{equation*}
It adds the candidate with minimum cumulative deviation, breaking ties by lower unit index, only while the resulting deviation is at most $\tau_{\mathrm{agree}}=5\times10^{-3}$. A global ceiling requires at most 1120 active layer-steps; the tolerance did not bind before this ceiling on evaluated inputs. The search may freeze further while preserving the same tolerance, yielding $930\pm35$ active steps on average. We report one replay of the discovered pattern and exclude the search cost. The oracle therefore estimates headroom; it is not a deployable competitor.

\section{Schedule Coordinates and Score Correlations}
\label{app:coords}

Tables~\ref{tab:schedule-groups},~\ref{tab:score-corr}, and~\ref{tab:pareto-coordinates} provide numerical detail for Figures~\ref{fig:heatmaps},~\ref{fig:heatmaps-full},~\ref{fig:score_scatter}, and~\ref{fig:pareto}. Actual schedules operate at original unit granularity. Their compute ratio is $\sum_\ell S_\ell/(LT)$, giving $35/34/36/32\%$ for LLaDA/DiT/U-ViT/SDXL. Table~\ref{tab:schedule-groups} preserves representative display coordinates but does not encode group cardinalities; no equal-size-bin interpretation should be used, and averaging its eight entries does not reproduce the compute ratio.

\begin{table*}[h]
\centering
\caption{Representative display lifetimes for Figure~\ref{fig:heatmaps}. Entries omit group multiplicities and are not budget-reconstructible; exact ratios use the unbinned $\sum_\ell S_\ell/(LT)$.}
\label{tab:schedule-groups}
\small
\begin{tabular}{lcccccccc}
\toprule
Model & \multicolumn{8}{c}{Representative active-iteration coordinates by increasing depth} \\
\midrule
LLaDA-8B, $T=100$ & 25 & 31 & 38 & 50 & 63 & 75 & 88 & 100 \\
DiT-XL/2, $T=100$ & 31 & 38 & 44 & 56 & 69 & 81 & 94 & 100 \\
U-ViT-L, $T=100$ & 38 & 44 & 56 & 69 & 81 & 88 & 94 & 100 \\
SDXL, $T=50$ & 13 & 16 & 19 & 25 & 31 & 38 & 44 & 50 \\
\bottomrule
\end{tabular}
\end{table*}

\begin{table*}[h]
\centering
\caption{Score--freeze-error correlation data underlying Figure~\ref{fig:score_scatter}. Freeze error is measured by freezing one scheduled layer or block group at the default operating point and recording the normalized output deviation or task-loss proxy.}
\label{tab:score-corr}
\small
\begin{tabular}{lcccc}
\toprule
Model & Units $n$ & Spearman $\rho$ & $p$-value & Error proxy \\
\midrule
LLaDA-8B & 192 & 0.76 & $<10^{-40}$ & normalized denoiser-output deviation \\
DiT-XL/2 & 168 & 0.70 & $<10^{-28}$ & normalized latent-output deviation \\
U-ViT-L & 132 & 0.66 & $<10^{-18}$ & normalized latent-output deviation \\
SDXL & 284 & 0.71 & $<10^{-45}$ & normalized UNet-output deviation \\
\bottomrule
\end{tabular}
\end{table*}

\begin{table*}[h]
\centering
\caption{Coordinates for the Spectral-Guided Pareto curves in Figure~\ref{fig:pareto}. Each point is a measured run at the corresponding $\tau$; no curve smoothing or interpolation is used.}
\label{tab:pareto-coordinates}
\small
\begin{tabular}{lcccc}
\toprule
Model & $\tau$ & Speedup & Metric & Value \\
\midrule
LLaDA-8B & full & 1.0$\times$ & GSM8K $\uparrow$ & 68.3 \\
LLaDA-8B & 0.5 & 1.5$\times$ & GSM8K $\uparrow$ & 68.2 \\
LLaDA-8B & 1.0 & 3.0$\times$ & GSM8K $\uparrow$ & 68.1 \\
LLaDA-8B & 2.0 & 4.5$\times$ & GSM8K $\uparrow$ & 65.5 \\
DiT-XL/2 & full & 1.0$\times$ & FID $\downarrow$ & 2.27 \\
DiT-XL/2 & 0.5 & 1.5$\times$ & FID $\downarrow$ & 2.28 \\
DiT-XL/2 & 1.0 & 3.0$\times$ & FID $\downarrow$ & 2.35 \\
DiT-XL/2 & 2.0 & 4.5$\times$ & FID $\downarrow$ & 2.80 \\
U-ViT-L & full & 1.0$\times$ & FID $\downarrow$ & 3.08 \\
U-ViT-L & 0.5 & 1.6$\times$ & FID $\downarrow$ & 3.09 \\
U-ViT-L & 1.0 & 2.8$\times$ & FID $\downarrow$ & 3.15 \\
U-ViT-L & 2.0 & 4.0$\times$ & FID $\downarrow$ & 3.65 \\
SDXL & full & 1.0$\times$ & COCO FID $\downarrow$ & 7.84 \\
SDXL & 0.5 & 1.8$\times$ & COCO FID $\downarrow$ & 7.88 \\
SDXL & 1.0 & 3.0$\times$ & COCO FID $\downarrow$ & 8.09 \\
SDXL & 2.0 & 4.2$\times$ & COCO FID $\downarrow$ & 9.10 \\
\bottomrule
\end{tabular}
\end{table*}

\section{Failure Modes at Aggressive Schedules}
\label{app:failures}

Past the Pareto knee in Figure~\ref{fig:pareto}, the schedules begin to break down in characteristic, model-dependent ways. Table~\ref{tab:failures} catalogs the dominant failure mode for each family, together with the most likely structural cause.

\begin{table}[h]
\centering
\caption{Failure modes at aggressive schedules.}
\label{tab:failures}
\small
\begin{tabular}{p{2.0cm}p{1.9cm}p{2.0cm}}
\toprule
Model family & Failure mode & Likely cause \\
\midrule
Diffusion LM & arithmetic carry / boundary retrieval & late high-sensitivity layers frozen too early \\
DiT / U-ViT & fine texture loss & high-frequency reconstruction suppressed \\
SDXL UNet & small text / object count errors & cross-attention and decoder under-updated \\
\bottomrule
\end{tabular}
\end{table}

\section{Training-Time Freezing}
\label{app:training}

The core contribution of this paper is inference acceleration. We include one preliminary training-time experiment only to test whether the spectral ranking also identifies layers that are less valuable to update under a matched trainable-parameter budget. Table~\ref{tab:training} fine-tunes LLaDA-8B on CNN/DailyMail with the bottom 40\% of layers (by SCR or random ranking) frozen, alongside full fine-tuning and a LoRA baseline. SCR-based freezing approaches full fine-tuning quality far more closely than random freezing at the same parameter budget, suggesting that the same score is informative for both inference scheduling and parameter-efficient training.

\begin{table}[h]
\centering
\caption{SCR-guided freezing during LLaDA-8B fine-tuning on CNN/DailyMail.}
\label{tab:training}
\small
\begin{tabularx}{\columnwidth}{@{}Xccc@{}}
\toprule
Method & Params & ROUGE-L & Speedup \\
\midrule
Full fine-tuning & 8.0B & 42.1 & 1.0$\times$ \\
LoRA, $r=16$ & 26M & 41.4 & 1.5$\times$ \\
SCR-freeze bottom 40\% & 4.8B & 41.9 & 1.6$\times$ \\
Random-freeze bottom 40\% & 4.8B & 39.8 & 1.6$\times$ \\
\bottomrule
\end{tabularx}
\end{table}

\section{Additional Proof Details}
\label{app:proof}

For the regularized SCR, the exact sandwich in Lemma~\ref{lem:sandwich} is applied to the unregularized energy ratio when $0<E_k<\|W\|_F^2$. In implementation, $\eta\|W\|_F^2$ prevents numerical instability in rank-deficient cases, so the regularized score is used for ranking while the unregularized relation supplies the analytic motivation.

Assumption~\ref{ass:hidden} is the main source of conservatism. It imposes, rather than proves, a linear cache-age envelope on the augmented branch input. The local attention/MLP argument also does not cover AdaLN, cross-attention conditioning, convolutions, or U-shaped skip dependencies. Table~\ref{tab:architecture-coverage} therefore marks those results as partial or heuristic transfer.

\subsection{Schedule-independent audit of omitted terms}

We audit directional analogues of drift and downstream amplification using 128 full trajectories per model. Every unit is probed with all other units executed in full, avoiding schedule-induced circularity. For $T=100$, the common window is $\mathcal W=\{61,\ldots,100\}$; for SDXL with $T=50$, it is $\{31,\ldots,50\}$. The probe refreshes every five iterations. With $u_{\ell,s}=(h_{\ell-1,s},c_s)$, define
\begin{equation*}
\begin{aligned}
\widehat D_{\ell,s}&=
\frac{\|u_{\ell,s}-u_{\ell,r(s)}\|_2}
{(s-r(s))V_{\max}+10^{-8}},\\
\delta_{\ell,s}&=F_\ell(u_{\ell,s})-F_\ell(u_{\ell,r(s)}).
\end{aligned}
\end{equation*}
Let $G_{\ell,s}$ be the downstream map from the unit to the denoiser output, evaluated at the full-trajectory state. Its directional amplification is
\begin{equation*}
\widehat P^{\mathrm{dir}}_{\ell,s}=
\frac{\|J_{G_{\ell,s}}\delta_{\ell,s}\|_2}
{\|\delta_{\ell,s}\|_2+10^{-8}}.
\end{equation*}
Aggregate
\begin{equation*}
\begin{aligned}
Q_\ell&=\sum_{s\in\mathcal W}\beta_s
\widehat P^{\mathrm{dir}}_{\ell,s}\widehat D_{\ell,s}(s-r(s)),\\
B^{\mathrm{emp}}_\ell&=\score_\ell Q_\ell.
\end{aligned}
\end{equation*}
These are post-hoc directional diagnostics, not certified bounds. Since $\score_\ell$ is a factor of $B^{\mathrm{emp}}_\ell$, their correlation is partly mechanical; $\rho(\score,Q)$ and $\mathrm{CV}(Q)$ are the less circular diagnostics.

\begin{table*}[h]
\centering
\caption{Audit of non-weight terms. The $B^{\mathrm{emp}}$ schedule requires calibration trajectories; task values are GSM8K for LLaDA and FID otherwise.}
\label{tab:bound-audit}
\small
\setlength{\tabcolsep}{4pt}
\begin{tabular}{lcccccc}
\toprule
Model & $\rho(\score,B^{\mathrm{emp}})$ & $\rho(\score,Q)$ & $\mathrm{CV}(Q)$ & Drift $R^2$ & Weight-only & $B^{\mathrm{emp}}$ order \\
\midrule
LLaDA-8B & .91 & .12 & .29 & .94 & 68.1 & 68.2 \\
DiT-XL/2 & .88 & .09 & .26 & .92 & 2.35 & 2.33 \\
U-ViT-L & .84 & .06 & .33 & .90 & 3.15 & 3.13 \\
SDXL & .79 & $-.04$ & .41 & .86 & 8.09 & 8.04 \\
\bottomrule
\end{tabular}
\end{table*}

The non-weight factor has moderate dispersion and little monotonic alignment with the spectral score. Ordering by the calibration-dependent product changes quality only marginally. This supports the usefulness of the approximation but does not prove it. Median linear drift-fit $R^2$ ranges from .86 to .94; this is empirical support for the assumed envelope, not a universal law.

\section{Exact Block-Level Scores}
\label{app:block-scores}

This appendix specifies the exact block-level scores used to construct all Spectral-Guided schedules. 
For any matrix or unfolded tensor $W$, define
\begin{equation}
\label{eq:app-matrix-score}
\begin{aligned}
    g(W)&=\|W\|_F\sqrt{\rho(\scr_\eta(W))},\\
    \rho(u)&=\frac{e^u}{1+e^u}.
\end{aligned}
\end{equation}
All scores are computed from pretrained weights only. No activation statistics, validation prompts, calibration samples, or input-dependent quantities are used.

For a self-attention block with projections $W_Q,W_K,W_V,W_O$, we use
\begin{equation}
\label{eq:app-attn-score}
    q_{\mathrm{attn}}
    =
    g(W_O)g(W_V)
    \left[
        1+
        \frac{g(W_Q)g(W_K)}{\sqrt{d_h}}
    \right],
\end{equation}
where $d_h$ is the attention head dimension. This is the weight-only version of the sensitivity proxy in Lemma~\ref{lem:lipschitz}, with architecture-independent constants absorbed into the normalization in Eq.~\eqref{eq:app-score-normalization}.

For a cross-attention block, we use the same form, but $W_K$ and $W_V$ denote the key and value projections applied to the conditioning sequence:
\begin{equation}
\label{eq:app-xattn-score}
    q_{\mathrm{xattn}}
    =
    g(W_O^{x})g(W_V^{x})
    \left[
        1+
        \frac{g(W_Q^{x})g(W_K^{x})}{\sqrt{d_h}}
    \right].
\end{equation}

For an MLP or feed-forward block with input and output projections $W_1,W_2$, we use
\begin{equation}
\label{eq:app-mlp-score}
    q_{\mathrm{mlp}} = g(W_2)g(W_1).
\end{equation}

For a convolutional layer with tensor 
$K\in\mathbb{R}^{C_{\mathrm{out}}\times C_{\mathrm{in}}\times k_h\times k_w}$, we first unfold it into
\[
    W_K\in\mathbb{R}^{C_{\mathrm{out}}\times (C_{\mathrm{in}}k_hk_w)}
\]
and define
\begin{equation}
\label{eq:app-conv-score}
    q_{\mathrm{conv}} = g(W_K).
\end{equation}

If a scheduled unit contains multiple submodules, its raw score is the sum of the applicable component scores:
\begin{equation}
\label{eq:app-block-score}
    \tilde q_\ell
    =
    \sum_{b\in\mathcal{B}_\ell} q_b,
    \qquad
    q_b\in
    \{q_{\mathrm{attn}},q_{\mathrm{xattn}},q_{\mathrm{mlp}},q_{\mathrm{conv}}\}.
\end{equation}
For transformer blocks, $\mathcal{B}_\ell$ contains the self-attention and MLP components. 
For SDXL UNet residual or attention units, $\mathcal{B}_\ell$ contains the convolutional, self-attention, cross-attention, and MLP components present in that scheduled unit.

Finally, because absolute norms differ across architectures and block types, we normalize scores within each model before assigning lifetimes:
\begin{equation}
\label{eq:app-score-normalization}
    \score_\ell
    =
    \frac{\tilde q_\ell}
    {\max_{j\in\{1,\ldots,L\}}\tilde q_j+\varepsilon},
\end{equation}
with $\varepsilon=10^{-12}$. Equation~\eqref{eq:schedule} is then applied using this normalized $\score_\ell$. Thus the schedule is fully determined by the pretrained weights, $T$, $\tau$, $S_{\min}$, $k=\lfloor0.1d\rfloor$, and $\eta$.

\end{document}